\documentclass{article}

\usepackage[final]{nips_2016}

\usepackage{amsmath,amsfonts,amsthm} 

\usepackage[utf8]{inputenc} 
\usepackage[T1]{fontenc}    
\usepackage{hyperref}       
\usepackage{url}            
\usepackage{booktabs}       
\usepackage{amsfonts}       
\usepackage{nicefrac}       
\usepackage{microtype}      
\usepackage{graphicx}
\usepackage{subcaption}
\usepackage{algorithmic}
\usepackage{algorithm}

\usepackage{listings}
\usepackage{lstlangAMPL}
\usepackage{mathabx}
\usepackage{thmtools}
\usepackage{mathtools}

\newcommand{\+}[1]{\ensuremath{\boldsymbol{#1}}}

\declaretheoremstyle[%
  spaceabove=-6pt,%
  spacebelow=6pt,%
  headfont=\normalfont\itshape,%
  postheadspace=1em,%
  qed=\qedsymbol%
]{mystyle} 

\newtheorem{theorem}{Theorem}

\newtheorem{proposition}[theorem]{Proposition}

\renewcommand{\qed}{\hfill \ensuremath{\Box}}

\newcommand{\calp}{{\cal P}}

\newcommand{\needcite}[1]{}

\newcommand{\be}{\begin{equation}}
\newcommand{\ee}{\end{equation}}
\newcommand{\nbe}{\begin{equation*}}
\newcommand{\nee}{\end{equation*}}
\newcommand{\bea}{\begin{eqnarray*}}
\newcommand{\eea}{\end{eqnarray*}}

\newcommand{\ignore}[1]{}

\renewcommand{\eqref}[1]{Eq.~\ref{#1}}

\usepackage{xcolor}

\definecolor{mygray}{rgb}{0.35,0.35,0.35}

\definecolor{mybgcolor}{HTML}{FFFED8}
\definecolor{mysumforallcolor}{HTML}{D45500}
\definecolor{darkgreen}{HTML}{1D676A}
\definecolor{darkred}{HTML}{FF2A6A}
\definecolor{darkvio}{HTML}{60146D}
\definecolor{mykeyword}{HTML}{0000FF}

\makeatletter
\newenvironment{proof*}[1][\bf\proofname]{\par
  \pushQED{\qed}%
  \normalfont \partopsep=\z@skip \topsep=\z@skip
  \trivlist
  \item[\hskip\labelsep
        \itshape
    #1\@addpunct{.}]\ignorespaces
}{%
  \popQED\endtrivlist\@endpefalse
}
\newenvironment{example*}[1][\bf Example]{\par
  \pushQED{\qed}%
  \normalfont \partopsep=\z@skip \topsep=\z@skip
  \trivlist
  \item[\hskip\labelsep
        \itshape
    #1\@addpunct{.}]\ignorespaces
}{%
  \popQED\endtrivlist\@endpefalse
}

\makeatother

\title{Lifted Convex Quadratic Programming}

\author{
   Martin Mladenov \\
  TU Dortmund Univeristy \\
 \texttt{martin.mladenov@cs.tu-dortmund.de} \\
   \And
  Leonard Kleinhans\\
  TU Dortmund Univeristy\\
   \texttt{leonard.kleinhans@tu-dortmund.de} \\
   \And
  Kristian Kersting \\
  TU Dortmund University \\
  \texttt{kristian.kersting@cs.tu-dortmund.de} \\
}

\begin{document}

\maketitle


\begin{abstract}
Symmetry is the essential element of lifted inference
that has recently demonstrated the possibility to perform 
very efficient inference in highly-connected, but symmetric 
probabilistic models models. This raises the question, 
whether this holds for optimisation problems in general.
Here we show that for a large class
of optimisation methods this is actually the case.
More precisely, we introduce the concept of fractional
symmetries of convex quadratic programs (QPs),
which lie at the heart of many machine learning approaches,
and exploit it to lift, i.e., to compress QPs.
These lifted QPs can then be tackled with the usual 
optimization toolbox (off-the-shelf solvers, cutting plane algorithms,
stochastic gradients etc.). If the original QP exhibits
symmetry, then the lifted one will generally
be more compact, and hence their optimization is likely to 
be more efficient. 
\end{abstract}

\section{Introduction}
Convex optimization is arguably one  of  the main motors behind the success of machine learning as it enables learning 
and inference in a wide variety of statistical machine learning models, such as 
SVMs and LASSO, as well as efficient approximations (e.g.~variational approaches, convex NMF) to hard inference tasks.
The language in which convex optimization problems are specified typically includes inequalities, matrix and tensor algebra,
and software packages for convex optimization such as CVXPY~\citep{cvxpy} recreate this language as an interface between the user and the solver.  
Unfortunately, these algebraic languages have one shortcoming: it is difficult---if not impossible---for the non-expert to 
directly make use of the discrete, combinatorial structure often underlying convex programs;
pixels depend only on neighboring pixels; quantities can flow only along specified links; the reward of placing a cup on a table does not depend on whether the window in the next room is open. Having a richer representation such as first-order logic to express the combinatorial structure and 
an automatic way to utilize it in the solver is likely extend the reach and efficiency of machine learning even further. 


This is akin to statistical relational learning (SRL) that has argued in favor of 
first-order languages when dealing with complex graphical models, see e.g.~\citep{deaedt16} for a recent overview. Moreover, due to the high-level nature of the 
relational probabilistic languages, the low-level (ground) model they produce might often contain redundancies in terms of symmetries: 
``indistinguishable'' entities of the model. Lifted probabilistic inference~\citep{poole03,deaedt16} approaches exploits these symmetries to perform very efficient inference in highly-connected (and hence otherwise often intractable for 
traditional inference approach) but symmetric
models. Intuitively,  one infers which variables are indistinguishable in the ground model (if possible without actually grounding) and solves the model treating the indistinguishable variables as groups instead of individuals. This dimensionality reduction is triggered by the knowledge of the high-level structure.
Unfortunately, SRL does not support convex quadratic optimization approaches commonly used in machine learning.

Here, we demonstrate that the core idea of SRL can be transferred to convex quadratic optimization.
As our main contribution, we  formalize
the notion of symmetries of convex quadratic programs (QPs). Specifically,
we first show that unlike for graphical models, where the notion of indistinguishability of variables is that of exact symmetry (automorphisms of the factor graph), QPs admit a weaker (partitions of indistinguishable variables which are at least as coarse) notion of indistinguishability called a fractional automorphism (FA)  resp.~equitable partition (EP). This implies that more general lifted inference rules for QPs can be designed. This 
is surprising, as it was believed that FAs apply only to linear equations. 
Second, we investigate geometrically how FAs of quadratic forms arise. The existing theory of symmetry in convex quadratic forms 
states that an automorphism of $\+x^TQ\+x$ corresponds to a rotational symmetry of the semidefinite factors of $Q$. We generalize this in that FA of $\+x^TQ\+x$ can be related not only to rotations, but also to certain scalings (as well as other not yet characterized properties of the semidefintie factors). This then results in
the first approximate FA approach based on standard clustering techniques and whitening.
Finally, we tackle the question to which extend kernels might preserve fractional symmetry. All this is embedded in a novel relational QP language,
which is not discussed due to space limitations. 

We proceed as follows. After reviewing prior art, we start developing automorphisms of QPs, introducing the required background on the fly. 
Then, we generalize this to fractional symmetries. Before concluding, we illustrate our theoretical results empirically. 

\section{Prior Art}
\label{related}
Several expressive modeling languages for mathematical programming have been proposed, see e.g.~\citep{wallace2005} for a recent overview. 
These modeling languages are mixtures of declarative and imperative programming styles using sets of objects to index 
multidimensional parameters and LP variables. Recently, 
\cite{cvxpy} enabled an object-oriented approach to constructing optimization problems.
However,  following \citet{logicblox}, one can still argue that
there is a need for 
languages 
that not only facilitates natural algebraic modeling 
but also provides integrated capabilities with logic programming. This is also witnessed by the growing need for relational mathematical
modeling e.g.~in natural language processing~\citep{roth2007,riedel12} and the recent 
push to marry statistical analytic frameworks like R and Python with relational databases~\citep{reABCJKR15}.
The present work is the first that introduces relational convex QPs and studies their symmetries. 
There are symmetry-breaking branch-and-bound approaches 
for {(mixed--)}integer programming~\citep{Margot_2010} that are also featured by commercial solvers. 
QPs, however, do not feature branch-and-bound solvers. 
For the special fragment of LPs, \citet{kersting2015} have introduced a relational language and shown how to exploit fractional symmetries. 
(Relaxed) graph automorphisms and variants have 
been explored for graph kernels~\citep{ShervashidzeB09} and (I)LP-MAP inference approaches~\citep{bui12arxive,mladenov14uai,jernite15icml}. Unfortunately,
their techniques or  proofs do not carry over to (convex) QPs.  \citet{GulerG12} and references in there have studied automorphisms 
but not fractional ones of convex sets.  Finally, our approximate FA approach generalizes \citeauthor{broeckD13}'s \citeyearpar{broeckD13} approach of approximating 
evidence in probabilistic relational models to QPs using real-valued low-rank factorizations.


\section{\bf Exact Symmetries of Convex Quadratic Programs}
Let us start off with exact symmetries of convex QPs. {\bf Lifting convex quadratic programs} essentially amount to reducing the size a model by grouping together ``indistinguishable'' variables and constraints. In other words, they exploit symmetries. To formalize the notion of lifting more concisely 
let us consider a {\bf convex program}, i.e., an optimization problem of the form 
\begin{equation}
\+x^* = \arg\min\nolimits_{\+x\in {\cal D}} J(\+x)\;, \tag{$\clubsuit$}
\end{equation}
over $\+x \in \mathbb{R}^n$, where $J:\mathbb{R}^n \rightarrow \mathbb{R}$ is a convex function, and ${\cal D}$ is a subset of $\mathbb{R}^n$, typically specified as the solution a system of convex inequalities $f_1(\+x) \leq 0, \ldots, f_m(\+x) \leq 0$. A {\bf convex quadratic program} (QP) is an instance of $(\clubsuit)$ where $J(\+x ) = \+x^TQ\+x + \+c^T\+x$ is a quadratic function with $Q \in \mathcal{R}^{n\times n}$ is symmetric and positive semi-definite, and ${\cal D} = \{\+x : A\+x \leq \+b\}$ is a system of linear equations. 
If $Q$ is the zero matrix, the problem is known as a {\bf linear program} (LP). If we add convex quadratic constrants to a quadratic program, we obtain a quadratically constrained quadratic program (QCQP). We will not deal explicitly with QCQPs in this paper, however, by the end of our discussion of quadratic functions, it will be evident that our results can easily be extended to such programs. We shall denote a QP by the tuple ${\+Q\+P} = (Q, \+c, A, \+b)$.   

We are now interested in partitioning the variables  of the program by a partition ${\cal P} = \{P_1,\ldots,P_p\}$, $P_i \cap P_j = \emptyset$, $\bigcup_i P_i = \{x_1,\ldots, x_n\}$, such that there exists at least one solution that {\bf respects} the partition. More formally, ${\cal P}$ is a {\bf lifting partition} of $(\clubsuit)$ if $(\clubsuit)$ admits an optimal solution with $x_i = x_j$ whenever $x_i$ and $x_j$ are in the same class in $\cal P$. We call the linear subspace defined by the latter condition $\mathbb{R}_{\cal P}$. 
Having apriori obtained a lifting partition of the QP, we can restrict the solution space to ${\cal D} {\cap \mathbb{R}}_{\cal P}$. That is, we constrain indistinguishable variables to be equal, knowing that at least one solution will be preserved in this space of lower dimension. Since ground variables of the same class are now equal, they can be replaced with a single aggregated (lifted) variable. The resulting lifted problem has one variable per equivalence class, thus, if the lifting partition is coarse enough, significant dimensionality reduction and in turn run-time savings can be achieved. To recover a ground solution, one assigns the value of the lifted variable to every ground variable in its class. 

One way to demonstrate that a given partition ${\cal P}$ is a lifing partition for $(\clubsuit)$ is by showing that {\bf averaging} any feasilbe $\+x$ over the partition classes (i.e. $\widetilde{x}_i = \frac{1}{|\operatorname{class}(x_i)|}\sum_{x_j \in \operatorname{class}(x_i)} x_j$) yields a new feasible $\widetilde{\+x}$ with $ J(\widetilde{\+x}) \leq J(\+x)$. As a consequence, by averaging any optimal solution we get another optimal solution which respects $\cal P$, implying that $\cal P$ is a lifting partition. One bit of notation that is handy in the analysis averaging operations is the {\bf partition matrix}. To any partition $\cal P$ we can associate a matrix $X^{\cal P} \in \mathbb{Q}^{n\times n}$ such that $X^{\cal P}_{ij} = 1/|\operatorname{class}(x_i)|$ if $x_j \in \operatorname{class}(x_i)$ or $0$ otherwise. With $X^{\cal P}$ defined thusly, averaging $\+x$ over the classes of $\cal P$ is equivalent to multiplying by $X^{\cal P}$, i.e., $\widetilde{\+x} = X^{\cal P}\+x$. Partition matrices are always {\bf doubly stochastic} ($X^{\cal P}\+1 = \+1$), {\bf symmetric} ($(X^{\cal P})^T=X^{\cal P}$), and {\bf idempotent} ($X^{\cal P}X^{\cal P}=X^{\cal P}$) -- as a consequence also {\bf semidefinite}.

\begin{figure}[t]
 
\centering
\begin{subfigure}{0.45\textwidth}
    \includegraphics[width=\textwidth]{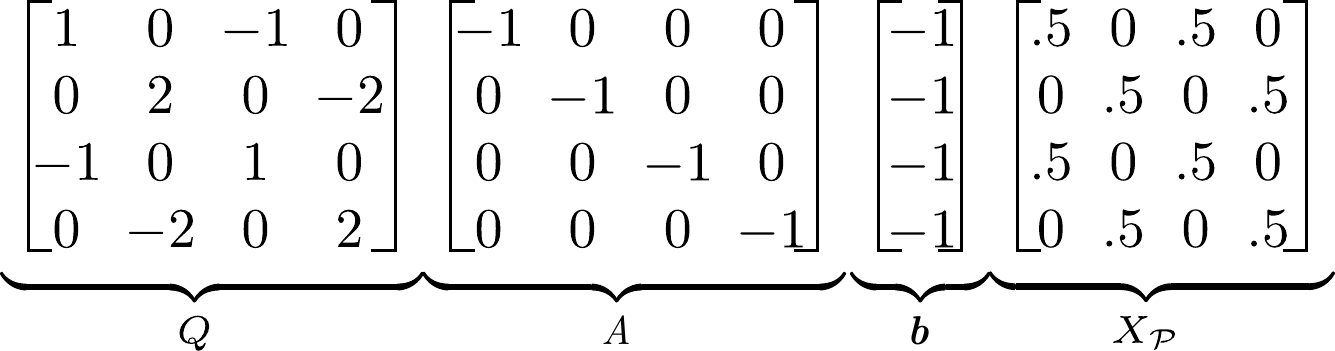}
    \caption{\label{fig:example_1}}
\end{subfigure}\quad\quad
\begin{subfigure}{0.45\textwidth}\centering
    \includegraphics[width=\textwidth]{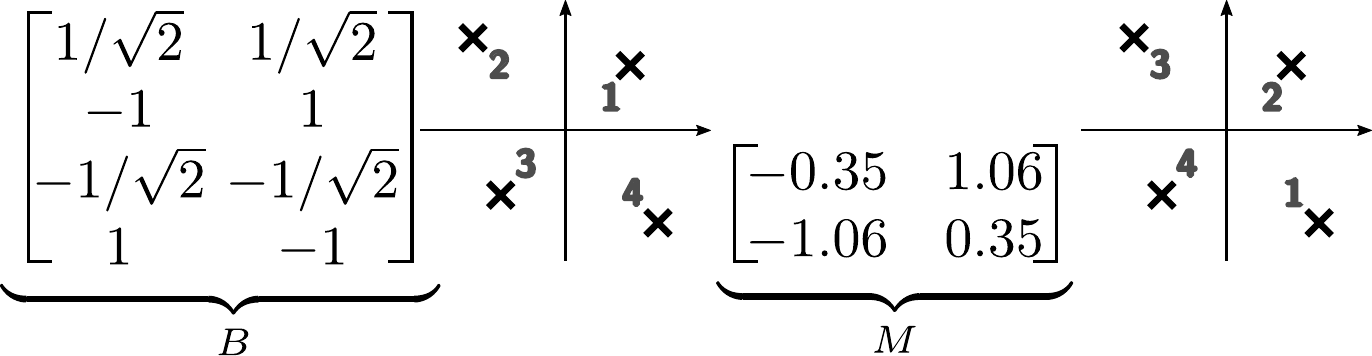}
    \caption{\label{fig:scaling}} 
\end{subfigure}
\caption{Running example for fractional symmetries of QPs. (a) A matrix specification of an example quadratic program $\operatorname{minimize}_{\+x \in {\mathbb{R}^4}}\+x^TQ\+x\; \operatorname{s.t.} A\+x\leq \+b$ and the partition matrix $X^{\cal P}$ of the partition ${\cal P} = \{\{x_1, x_3\},\{x_2, x_4\}\}$. (b) The factor $B$ with $BB^T = Q$ relating to part (a) as well as a sketch of the rows of $B$. Multiplying $B$ by the matrix $M$ on the right, which equates to rescaling and rotating the vectors by $45^\circ$ is a symmetry of $B$ as it yields the same configuration modulo renaming.}
\end{figure}

\begin{example*}
We seek to minimize the function $\+x^TQ\+x$ over $\+x \in \mathbb{R}^4$, subject to $\+x \geq 1$, with $Q$ given in Fig.~\ref{fig:example_1}. As a lifting partition, we propose ${\cal P} = \{\{x_1, x_3\}, \{x_2, x_4\}\}$ (in the next paragraph, we will explain how one could compute this lifting partition). The corresponding parition matrix $X^{\cal P}$ is also shown on Fig.~\ref{fig:example_1}. Let us demonstrate that averaging over the classes of $\cal P$ decreases the value of the solution. For example, for $\+x_0 = [2,1,1,2]^T$, $\+x_0^TQ\+x_0 = 3$. On the other hand, the class-averaged $\widetilde{\+x}_{0} = X_{\cal P}\+x_0 = [1.5, 1.5, 1.5, 1.5]^T$ yields a value of $0$. In fact, one could notice that any feasible $\+x$ respecting the partition yields a value of $0$, so any such solution is optimal. Moreover, if all coordinates of $\+x$ are already greater than or equal to $1$, then the same holds for $\widetilde{\+x}$, as averages cannot be lower than the minimum of the averaged numbers. Thus, the compressed problem reduces to finding any two numbers that greater than or equal to $1$. In a sense, lifting solves this problem without having to resort to numerical optimization.   
\end{example*}

An intuitive way to find lifting partitions is via {\bf automorphism groups} of convex problems. We define the automorphism group of $(\clubsuit)$, $\operatorname{Aut}(\clubsuit)$, as the group of all pairs of permutations $(\sigma, \pi)$ with permutation matrices $(\Sigma, \Pi)$, such that for all $\+x$, $J(\+x) = J(\Pi\+x)$ and $(f_1(\Pi\+x) \leq 0, \ldots, f_m(\Pi\+x) \leq 0) = (f_{\sigma(1)}(\+x) \leq 0, \ldots, f_{\sigma(m)}(\+x) \leq 0)$. In other words, {\bf renaming} the variables yields the same constraints up to reordering. For linear programs (LPs), this is equivalent to $\Sigma A = A\Pi$ and $\Sigma \+b = \+b$ and $\+c^T\Pi = \+c^T$. The partition that groups together $x_i$ with $x_j$ if some $\Pi$ in $\operatorname{Aut}(\clubsuit)$ exchanges them is called an {\bf orbit partition}. An interesting fact is that if $\cal P$ is an orbit partition, $X_{\cal P}$ is the {\bf symmetrizer matrix} of $\operatorname{Aut}(\clubsuit)$,  $X_{\cal P} = \frac{1}{|\operatorname{Aut}(\clubsuit)|}\sum_{(\Sigma, \Pi) \in \operatorname{Aut}(\clubsuit)} \Pi$. One way to detect renaming symmetries is by inspection of the parameters of the problem. E.g., for a convex quadratic program $(Q,\+c,A,\+b)$, a set of necessary conditions for the pair of permutations $(\Sigma, \Pi)$ to be a renaming symmetry is: (i) $\Pi Q = Q \Pi$ (equivalently $\Pi Q \Pi^T = Q $), (ii) $\+c^T\Pi = \+c^T$, (iii) $\Sigma A = A\Pi$, and (iv) $\Sigma \+b = \+b$. Such automorphism groups, or rather, the orbit partitions thereof, can be computed via packages such as Saucy~\cite{codenottiKSM13}. The reason why orbit partitions are lifting partitions of a convex problem, is that $J(X^{\cal P}\+x) = J(\frac{1}{|\operatorname{Aut}|}\sum\nolimits_{(\Sigma, \Pi) \in \operatorname{Aut}}\Pi\+x) \leq \frac{1}{|\operatorname{Aut}|}\sum\nolimits_{(\Sigma, \Pi) \in \operatorname{Aut}}J(\Pi\+x) = J(\+x)$, the inequality being due to convexity of $J$. Recalling our example on Fig.~\ref{fig:example_1}, we notice that permutations renaming row/column $1$ to $3$ resp. $2$ to $4$ are automorphisms, and our proposed $\cal P$ is an orbit partition.   


For the special case of LPs, 
\citet{GroheKMS14} have proven that 
equitable partitions 
act as lifting partitions. An {\bf equitable partition} of a square symmetric ${n\times n}$ matrix $M$ is a partition ${\cal P}$ of $1,\ldots,n$, such that $X^{\cal P}$ satisfies $X^{\cal P}M = MX^{\cal P}$. For rectangular matrices, we say that a partition $\cal P$ of the columns is equitable, if there exists a partition of the rows $\cal Q$ such that $X^{\cal Q}M = MX^{\cal P}$. For LPs, we say that a partition of the variables $\cal P$ is equitable if there exists a partition of the constraints $\cal Q$ such that: $\+c^TX^{\cal P}=\+c^T$, $X^{\cal Q}\+b = \+b$, and $X^{\cal Q}A = AX^{\cal P}$. Equitable partitions and their corresponding partition matrices are refered to as {\bf fractional automorphisms} or {\bf fractional symmetries}, as they satisfy the same conditions as automorphisms from the previous paragraph, except that $X^{\cal P}$ is a doubly stochastix matrix and not a permutation matrix. 
Moreover, equitable partitions have an equivalent combinatorial characterization. A partition $\cal P$ of $M\in{n\times n}$ is equitable if for all $i,j$ in the same class $P$ and every class $P^\prime$ (including $P^\prime = P$), we have $\sum_{k \in P^\prime} M_{ik} = \sum_{k \in P^\prime} M_{jk}$. In other words, if we reorder the rows and columns of $M$ such that indices of the same class are next to eachother, $M$ will take on a block-rectangular form where every row (and column) of the block has the same sum. One special flavor of equitable partitions are what we will call {\bf counting partitions}, where a narrower condition holds, $|\{k \in P^\prime | M_{ik} = c\}| = |\{k \in P^\prime | M_{jk} = c\}|$ for all $c\in\mathbb{R}$, and $M_{ii} = M_{jj}$ if $i,j$ are in the same class. They partition $M$ into blocks where each row (and column) have the same count of each number. The equitable partition of our example is such a partition. In fact, any orbit partition of a permutation group is a counting partition as well. 
Equitable partitions have several very attractive properties when used as lifting partitions. First, the coarsest equitable partition (as well as the coarsest counting equitable partition) of a matrix is computable in ${\cal O}((e + n)\log(n))$ time, where $e$ is the number of non-zeroes in the matrix, via an elegant algorithm called color refinement. Second, the coarsest equitable partition is at least as coarse as the orbit partition of a matrix, hence it offers more compression. 

\section{Fractional Symmetry of Convex Quadratic Programs}


Having developed automorphisms of convex QPs, we now move on to our main contributions. We develop
FA esp.~EPs of a convex QP.
We start off with showing that they are lifted partitions. Then, we 
provide a geometric interpretation and investigate whether . 
kernels preserve fractional symmetries. 

{\bf Equitable Partitions of Quadratic Programs:}
We start be proving that the lifting partition of a convex QP captures its symmetries.
\begin{theorem}
\label{thm:lift}
Let $\+Q\+P = (Q, \+c, A, \+b)$ be a convex quadratic program.  If ${\cal P}$ is a partition of the variables of $\+Q\+P$, such that: (a) $X^{\cal P}Q = QX^{\cal P}$ and $\+c^TX^{\cal P} = \+c^T$, (b) there exists a partition ${\cal Q}$ of the constraints of $\+Q\+P$ such that $X^{\cal Q}\+b = \+b$ and $X^{\cal Q}A = AX^{\cal P}$, then $\cal P$ is a lifting partition for $\+Q\+P$. 
\end{theorem}

\begin{proof*}
We proceed along the lines drawn out in the previous section and show that for any feasible $\+x$, $\+x^\prime = X^{\cal P}\+x$, the class-averaged $\+x$, is both feasible and $J(\+x^\prime) \leq J(\+x)$.
Let us start with the latter. Note that both $Q$ and $\+X^{\cal P}$ are diagonalizeable (i.e. admit an eigendecomposition). It is known that if two diagonalizeable matrices commute (as is our starting hypothesis, $X^{\cal P}Q = QX^{\cal P}$), then they are also simultaneously diagonalizeable. That is, there exists an orthonormal basis of vectors $\+u_1,\ldots,\+u_n$ such that $Q = \sum_i \lambda_i \+u_i \+u_i^T = U\Lambda U^T$ and $X^{\cal Q} = \sum_i \kappa_i \+u_i \+u_i^T = UKU^T$, where the $\lambda_i$'s and $\kappa_i$'s are nonnegative scalars. Now,
$J(\+x^\prime) = J(X^{\cal P}\+x) = \+x^T(X^{\cal P})^T Q X^{\cal P}\+x + \+c^TX^{\cal P}\+x$. 
From our discussion so far and assumption (a), this is equal to $\+x^TU K^T U^T U \Lambda U^T U K U^T\+x + \+c^T\+x = \+x^T U\Lambda K^2 U^T\+x + \+c^T\+x$. The key observation is that because $X^{\cal P}$ is doubly stochastic, $|\kappa_i| \leq 1$. Hence $\+x^TU K^2 \Lambda\+x = \sum_{i} \kappa_i^2\lambda^i \+x^T\+u_i\+u_i^T\+x \leq \sum_{i} \lambda^i \+x^T\+u_i\+u_i^T\+x$ as $\lambda_i\+x^T\+u_i\+u_i^T\+x$ is a nonnegative quantity. This entails $J(\+x) \geq J(\+x^\prime)$.  

Regarding feasibility, because ${X^{\cal Q}}$ is a matrix of nonnegative numbers, $A\+x \leq \+b$ implies $X^{\cal Q}A\+x \leq X^{\cal Q}\+b$. Due to 
(b), this becomes $AX^{\cal P}\+x \leq\+b$, that is, $A\+x^{\prime} \leq\+b$, demonstrating the feasibility of $\+x^{\prime}$.

We have thus satisfied the two sufficient conditions stated in the previous section and shown that any $\cal P$ satisfying our assumptions is a lifting partition for $\+Q\+P$.
\end{proof*}
\begin{example*}
Recall $Q$ from our running example on Fig.~\ref{fig:example_1}.
However, this time we propose ${\cal P}^\prime = \{x_1, x_2, x_3, x_4\}$ as a lifting partition with $X^{{\cal P}^\prime} = \frac{1}{4}\cdot 1_4$, where $1_4$ is the $4\times 4$ matrix of ones. We observe that $Q X^{{\cal P}^\prime} = X^{{\cal P}^\prime}Q  = 0_4 $, moreover, if we introduce the constraint partition ${\cal Q}^\prime = \{y_1, ..., y_4\}$ with partition matrix $X^{{\cal Q}^\prime} = X^{{\cal P}^\prime}$, we have that $X^{{\cal Q}^\prime}A = AX^{{\cal P}^\prime}$ and $X^{{\cal Q}^\prime}\+b = \+b$. According to Thm.~\ref{thm:lift} ${\cal P}^\prime$ is a lifting partition of the QP in question. 
\end{example*}

There are two interesting observations to be made here. First, we have gained even further compression over our previous attempt, having a compressed problem with $1$ variable instead of $2$. Second, there is no automorphism of $Q$ that could possibly exchange $x_1$ and $x_2$. As fractional symmetries generalize exact symmetries, it is to be expected that coarser equitable partitions than the orbit partition $Q$ could satisfy the conditions of Thm~\ref{thm:lift}. 
Moreover, these observations allow one to gain insight into what fractional symmetry means geometrically for a dataset. This is important as
the matrix $Q$ relates to the data we feed into the optimization problem for many QPs. 
For example, in the SVM dual quadratic program, the entries of $Q$ are inner products of the feature vectors of the training examples.

{\bf Geometry of Fractionally-Symmetric QPs:}
Our investigation is inspired by the characterization of automorphisms of semidefinite matrices and quadratic forms. One way to think about a semidefinite matrix $Q$ is as the Gram matrix of a set of vectors, i.e. $Q = BB^T$ where $B$ is an ${n\times k}$ matrix and $k \geq \operatorname{rank}(Q)$. 
In this light, the quadratic form $\+x^T Q\+x$ can be seen as the squared Euclidean norm of a matrix-vector product. That is, $\+x^T Q\+x = \+x^T BB^T\+x =  (B^T\+x)^T(B^T\+x) = ||B^T\+x||^2$. It is a basic fact that the Euclidean norm is invariant under orthonormal transformations, that is, for any orthonormal matrix $O$ and any vector $\+y$, $||O^T\+y||=\+y$ as $\+y^TOO^T\+y = \+y^T\+y$. Thus, suppose we have a {\bf rotational autmorphism} of $B$, i.e., a pair of orthonormal matrix $O$ and permutation matrix $\Pi$, such that $\Pi B = BO$ or also $\Pi B O^T = B$. That is, rotating the tuple of vectors that are the rows of $B$ together yields same tuple back, but in different order. Observe then, that $\Pi$ would be a renaming automorphism for $Q$, since $\Pi Q\Pi^T = \Pi B O^T O B \Pi^T = BB^T = Q$, implying $\Pi Q = Q\Pi$. Moreover, if the right dimension (number of columns) $B$ is held fixed, the converse is true as well \cite{Bremner09}. That is, not only do rotational symmetries of $B$ correspond to renaming symmetries of $Q$, but vice-versa, as for fixed $k$, the semidefinite factors of $Q$ are unique up to rotations.
\begin{example*}
Our $Q$ from Fig.~\ref{fig:example_1} can be factored into $BB^T$ as shown on Fig.~\ref{fig:scaling}. The Figure also shows the plot of these vectors. If we were to rotate them by $180^\circ$ counter-clockwise, we would get back the same set of vectors, but in the order $\{x_3, x_4, x_1, x_2\}$. The permutation matrix according to this reordering is a renaming automorphism of $Q$.  
\end{example*}
Using the case of automorphisms as a motivation, we now turn to fractional automorphisms. More precisely, given a doubly stochastic and idempotent matrix $X$, such that $XQ = QX$, we would like to derive a similar characterization of $X$ in terms of $B$. As we prove now, this is indeed possible. 
\begin{theorem}
Let $X$ be a symmetric and $X$ is idempotent (as our usual color-refinement automorphisms are) matrix, and $Q = BB^T$ be a positive semidefinite matrix with $B$ having full column rank. Then $XQ = QX$ \emph{if and only if} there exists a symmetric matrix $R$ such that $XB = BR$.
\label{thm:bchar}
\end{theorem}
\begin{proof*}
({\bf only if direction}): Suppose there exists an $R$ such that $R = R^T$ and $XB = BR$ . Then, 
$XQ = XBB^T = BRB^T\;.$ Making use of $R = R^T$ this rewrites as $BR^TB^T = B(BR)^T = B(XB)^T = BB^TX^T = QX\;,  $
as $X$ is also symmetric. 

({\bf if direction}): Let $XQ = QX$ with $X$ being idempotent and symmetric. Then, let $R = B^TXB(B^TB)^{-1}$. Observe that $B(B^TB)^{-1}$ exists and is the right pseudoinverse of $B^T, i.e.,  B^TB(B^TB)^{-1} = I_k,$ as $B$ has full column rank.  Therefore, left multiplying by $I_k$ yields
$XB = XB B^TB(B^TB)^{-1} =  BB^TXB(B^TB)^{-1} = BR\;.$
It remains to demonstrate that $R$ is symmetric. Recall that $R^TR$ and $(R^TR)^{-1}$ are symmetric matrices. Then, 
$R^TR = \left[B^TXB(B^TB)^{-1}\right]^TB^TXB(B^TB)^{-1}\;.$ Since, $(B^TB)^{-T}= (B^TB)^{-1}$ and $XBB^T=BB^TX$, this simplifies to 
$(B^TB)^{-1}(B^TB)XXB(B^TB)^{-1}$. Since $XX=X$ and using $I_k$, this simplifies to 
$B^TXB(B^TB)^{-1} = R\;.$
Hence, as $R^TR = R$, $R$ is symmetric.
\end{proof*}

This theorem holds the key to explaining why all $4$ dimensions in our example are compressed together. To see why, consider the situation on Fig~\ref{fig:scaling}. 
\begin{example*}
Fig~\ref{fig:scaling} shows the factor $B$ of $Q$ (as well as a sketch of its rows). It also shows an invertible matrix $M$, which consists of a clockwise rotation by $45^\circ$ which aligns the vectors with the axes, a rescaling of the vectors along the axes, then a further $45^\circ$. Multiplying $B$ by this matrix yields back the same row vectors modulo a cyclic permutation, exchanging $x_4$ with $x_1$, $x_1$ with $x_2$ and so on, i.e. $\Sigma B = BM$. Moreover $BMM^TB^T\neq Q$. The group of $\{M, M^2, M^3, M^4\}$ is thus a group that does not correspond to any group of automorphisms of $Q$, yet, the symmetrizer matrix $\frac{1}{4}\sum_{i=1}^4 M^i$ is symmetric (and equal to $0_2$), so it qualifies under the conditions of Thm.~\ref{thm:bchar}.  
\end{example*}


From this we can conclude that certain scaling symmetries of $B$ do not result in symmetries of $Q$, but do result in {\bf fractional symmetries} of $Q$ (Thm.~\ref{thm:bchar} ). On the other hand, by Thm.~\ref{thm:lift}, we can also infer that these symmetries can safely be compressed out when minimizing the quadratic form $\+x^TQ\+x$. Note finally that even these symmetries do not exhaust the possible matrices of Thm.~\ref{thm:bchar} -- Thm.~\ref{thm:bchar} allows for partitions and matrices that do not correspond to any group. 
Characterizing them is an exciting avenue for future work.

\begin{figure*}[t]
\centering
\includegraphics[width=0.16\textwidth]{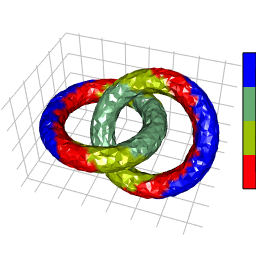}
\quad
\includegraphics[width=0.16\textwidth]{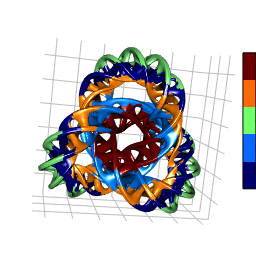}
\quad
\includegraphics[width=0.16\textwidth]{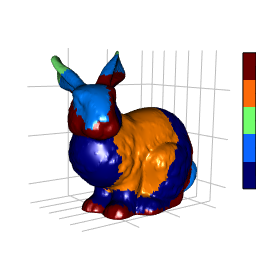}
\quad
\includegraphics[width=0.16\textwidth]{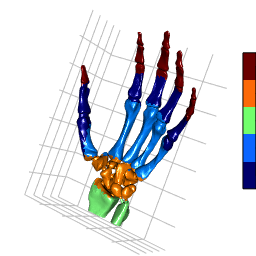}
\quad
\includegraphics[width=0.16\textwidth]{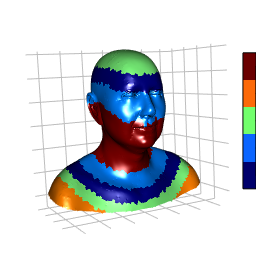}

\caption{Approximate EPs without scalings on 3D datasets. The colors encode the rotational symmetries under a budget of 4 resp.~5 orbits.  The largest dataset is the "hand" with $327.323$ points (a clique with $50\cdot 10^{10}$ edges) running in $<5$ secs using $2500$ anchor points. (Best viewed in color)\label{approx}}
\end{figure*}
Unfortunately, the (rotational) automorphism group of most Euclidean datasets consists of the identity transformation alone. This follows from 
the same result for convex bodies, see e.g. \citep{GulerG12}, and is to be expected, since the symmetry properties of a
given dataset $B$ can easily be destroyed by slightly perturbing the body. To bypass this, we propose the first {approximate lifting} approach for 
Euclidean datasets. 
\begin{proposition}
Let $B$ be an Euclidean dataset and $D$ its corresponding pairwise distance matrix. Then $B_{i\cdot}$ and $B_{j\cdot}$ are in the same (rotational) orbit 
if an only if $B_{i\cdot}$ and $B_{j\cdot}$ have the same sorted distances to all other data points. 
\label{thm:approx}
\end{proposition}
\begin{proof*}
The EP of $D$ encodes the symmetries of $B$. To compute it, we represent it as a colored graph $C$ of $D$. We note that $C$ is a clique 
with edge colors encoding distances.  We turn this into a node-colored graph by assigning the same color to all nodes that have identical edge-color signatures.  
Runing color-refinement on this graph does not add any new color since since $C$ is a clique.
\end{proof*}
This suggest a simple way to compute proper approximations of (rotational) EPs of $B$: (1, optional) Whiten the data to capture some scalings, (2) compute the pairwise distance matrix $D$ of $B$ (potentially using anchor points), (3)
sort each row of $D$, and (4) run any cluster algorithm on the sorted distance matrix. This is illustrated in Fig.~\ref{approx} and should 
be explored further in future work.

{\bf Kernels and Equitable Partitions:}
Finally, we touch upon the relationship between the fractional symmetry of data vectors and kernels. Kernel functions often appear in conjunction with quadratic optimization in machine learning problems as a means of enriching the hypothesis space of a learner. From an algebraic perspective, the essence approach is to replace the entries of the semidefinite matrix $Q$ with the values of a kernel function, which represents the inner product of data vectors under some non-linear transformation in a high dimensional space. That is, in place of $Q_{ij} = \left<B_{\cdot i}, B_{\cdot j} \right>$, we use $K_{ij} = k(B_{\cdot i}, B_{\cdot j}) = \left<\phi(B_{i\cdot}), \phi(B_{j\cdot}) \right>$, where $\phi: \mathbb{R}^n \rightarrow \mathbb{R}^m$ is some non-linear function with $m$ much greater than $n$ or even infinite. Due to the prevalence of kernels, it is important to understand whether kernels preserve or destroy symmetries. 
Here, we will examine two popular kernels, the polynomial kernel, $k^{\mathrm{POLY}}(\+x, \+y) = (\left< \+x, \+y \right> + 1)^g$ and $k^{\mathrm{RBF}}(\+x, \+y) = \exp(-2\gamma^2||\+x - \+y||^2_2)$, where $g$ is a positive integer and $\gamma$ is a nonzero real number. We find that in both cases, if $Q = BB^T$ admits a {\bf counting equitable partition}, then $K$ will admit the same partition as well, i.e., these two kernels preserve fractional symmetry of $Q$ up to counting (recall, that includes rotational symmetry of $B$): 
\begin{proposition}
Let $B$ be a matrix whose rows are data instances. Then, if $Q = BB^T$ admits a counting equitable partiton ${\cal P}$ with partition matrix $X^{\calp}$, then both kernel matrices {\bf (a)} $K^{\rm POLY}$  and {\bf (b)} $K^{\rm RBF}$  of this set of vectors admit the same counting partition. 
\end{proposition}
\begin{proof*}
Recall that an equitable partition $\cal P$ is a counting partition for $Q$ if for all $x_i, x_j$ in the same class $P$, and for every class $P^\prime$ (including $P^\prime = P$), $|\{x_k \in P^\prime | Q_{ik} = c\}| = |\{x_k \in P^\prime | Q_{jk} = c\}|$ for all $c\in\mathbb{R}$, and $Q_{ii} = Q_{jj}$. 
{\bf (a)}  A direct consequence of this definition is that if $\cal P$ is a counting partition for $Q$, it will be a counting partition for every other matrix whose equality pattern respects that of $Q$, in other words, $Q_{ij} = Q_{pq} \Rightarrow K_{ij} = K_{pq}$. $K^{\rm POLY}$ has exactly this property: $K^{\rm POLY}_{ij} = (\left< B_{i\cdot}, B_{j\cdot} \right> + 1)^g = (Q_{ij} + 1)^g$. It is clear that if $Q_ij$ and $Q_pq$ are equal, the values of the last expression would be equal as well. 
{\bf (b)} First, we note
$K^{\rm RBF}_{ij}  = \exp(-2\gamma^2||B_{i\cdot}||^2)$$\exp(-2\gamma^2||B_{j\cdot}||^2)$$\exp(-\gamma^2\left< B_{i\cdot}, B_{j\cdot} \right>)$. This allows one to rewrite
$K^{\rm RBF}$ in terms of $Q$:  $K^{\rm RBF}_{ij} =  \exp(-2\gamma^2Q_{ii})\exp(-2\gamma^2Q_{jj})\exp(-\gamma^2Q_{ij})$. Now, let $x_i, x_j \in P$ and $x_p, x_q \in P^\prime$ such that $Q_{ip} = Q_{jq}$. Since $Q_{ii} = Q_{jj}$ (by virtue of being in $P$) and $Q_{pp} = Q_{qq}$ (by virtue of $P^\prime$), we have that $K^{\rm RBF}_{ip} = K^{\rm RBF}_{jq}$ hence counts across classes are preserved. 
\end{proof*}

To summarize, in order to lift a convex QP, we compute its quotient model w.r.t its EP as illustrated in Fig.~\ref{liftingQPs}. For the two popular kernels---polynomial and RBF---this  also leads to valid liftings.

\begin{figure*}[t]
\centering
\includegraphics[width=0.32\textwidth]{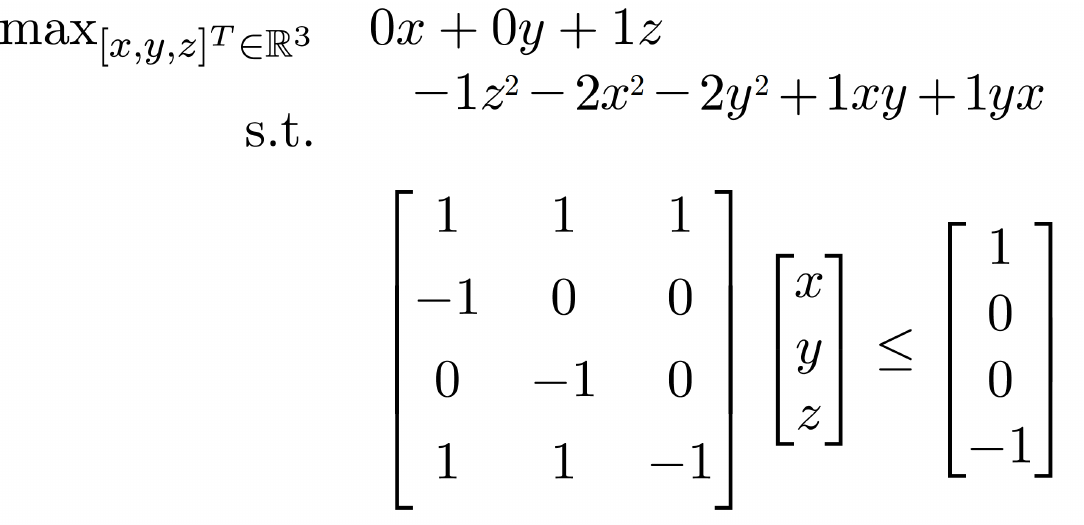}
\quad\quad\quad\quad\quad\quad
\includegraphics[width=0.40\textwidth]{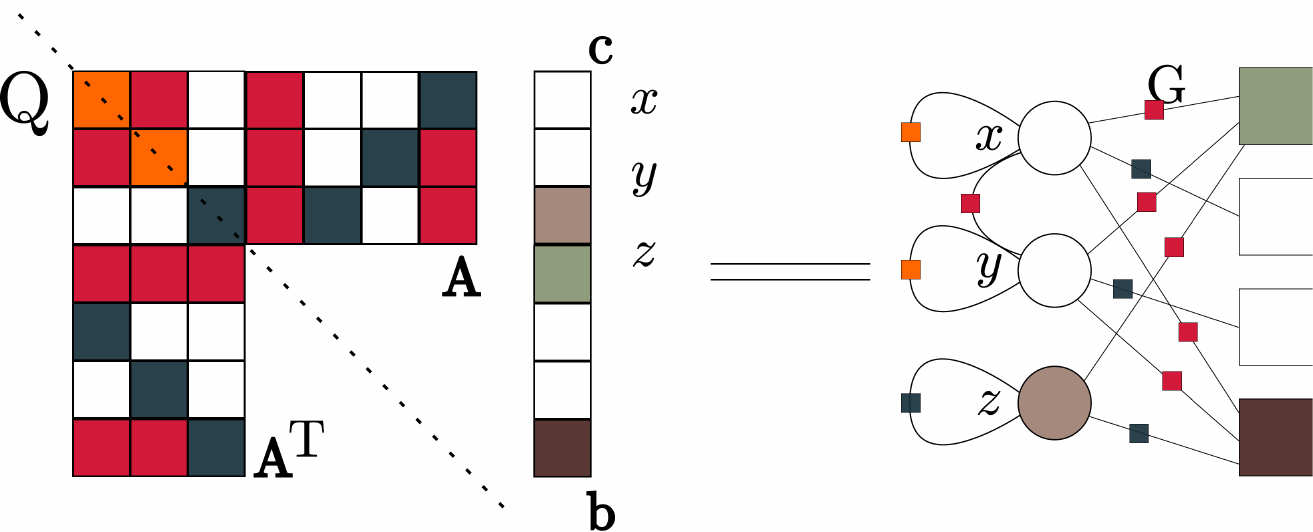}
\caption{Our theoretical results prove a QP (left) can be lifted using color-refinement, see e.g.~\citep{kersting2015}. It is first encoded as a colored graph (right). After running color-refinement on the graph, nodes with the same color form the quotient QP of the original QP. (Best viewed in color)\label{liftingQPs}}
\end{figure*}

\section{Empirical Illustration}

Our intention here is to investigate the following question {\bf (Q)}:  Can machine learning problems potentially benefit from fractional symmetries of QPs?
Generally this is to be expected e.g.~for classification as if all the data points of an orbit share the same label, then this symmetry effectively lowers the VC-dimension and sample complexity of the classifier~\citep{abu-Mostafa93}.

In a first experiment, we considered SVM classifiers for varying amounts of overlap between two classes represented by spherical Gaussians.
This dataset was chosen in order to depict the potential of approximate symmetries. We trained 
a lifted SVM (LSVM) with $200$ approximate color classes and a conventional SVM, both with RBF kernels, on $2500$ training examples per class.
We used a grid search together with CV for selecting $\gamma=\{0.25, 0.50, 1.00, 2.00, 4.00\}$ and $C=\{0.5,1.0,2.0 \}$. The performance was measured
on an independently drawn test set of $5000$ data points per class. For approximate lifting we used k-Means using the Euclidean metric and $500$ anchor points.  
For $4$ units apart class centers, 
the SVM achieved an error of $0.02$ in $20$ secs (all numbers in this experiment are averaged over 10 reruns and rounded to the second digit), 
while the LSVM achieved $0.02$ in $1.7$ secs.  An SVM using just the anchor points as training set achieved an error of $0.06$ in $2.1$ secs.
For closer class centers, namely $2$ units apart, 
the SVM took $98$ secs achieving an error of $0.16$, while the LSVM achieved $0.17$ in $2.1$ secs. 
The ''anchor'' SVM achieved an error of $0.16$ in $2.1$ secs.

In a second experiment, we considered a relational classification task on the Cora dataset~\citep{senNBGGE08}. The Cora dataset consists of $2708$ scientific papers classified into seven classes. Each paper is described by a binary word vector indicating the absence/presence of a word from a dictionary of $1433$ words. 
The citation network of the papers consisting of $5429$ links. The goal is to predict the class of the paper. For simplicity, we converted this problem to a binary classification problem by taking the largest of the $7$ classes as a positive class. 
We compared four different learners on Cora. The base classifiers are an $\infty$-norm regularized SVM (LP-SVM)~\citep{zhou2002linear}
and a conventional SVM (QP-SVM)~\citep{vapnik1998statistical} formulated as a convex QP. 
Both use the word feature vectors and do standard linear prediction (no kernel used). Additionally we considered
transductive, collective versions of both of them following~\citet{kersting2015}, denoted as TC-LP-SVM resp.~TC-QP-SVM.
Both transductive approaches have access to the citation network and implement the following simple rule: whenever we have access to an unlabaled paper $i$, if there is a cited or citing labeled paper $j$, then assume the label of $j$ as a label of $i$. To account for contradicting constraint (a paper citing both papers of and not of its class), we introduced separate slack variables for the transductive constraints and add them to the objective with a different penalty parameter. 
This can easily be implemented by adding a few lines to an existing standard QP-SVM formulation as illustrated in Fig.~\ref{rlp:tclpsvm}.  
In order to investigate the performance, we varied the amount of labeled examples available. That is, we have four cases, where we restricted the amount of labeled examples to $t = 20\%$, $40\%$, $60\%$, and $80\%$ of size of the dataset. We first randomly split the dataset into a labeled set $L$ and an unlabeled test set $B$, according to $t$. Then, we split $L$ randomly in half, leaving one half for training - $A$, the other half becoming a validation set $C$. The validation set was used to select the parameters of the TC-QP-SVM in a $5$-fold cross-validation fashion. That is, we split the validation set into 5 subsets $C_i$ of equal size. On these sets we selected the parameter using a grid search for each $C_i$ on a $A \cup (C \setminus C_i)$ labeled and $B \cup C_i$ unlabeled examples, computing the prediction error on $C_i$ and averaging it over all $C_i$s. We then evaluated the selected parameters on the test set $B$ whose labels were never revealed in training. 
We repeated this experiment $5$ times (one for each $C_i$) for the TC-SVMs. For consistency, we followed the same protocol with QP-SVM and LP-SVM, except that the set $B \cup C_i$ did not appeared during training as the non-transductive learners have no use for unlabeled examples.
That is, we selected parameters by training on $A \cup (C \setminus C_i)$ and evaluating on $C_i$. The selected parameters were then evaluated on the test set $B$.
For all SVM models, we also ran a ground and a lifted version. 
The results are summarized in Fig.~\ref{fig:result_1}. The QP-SVM outperforms the LP-SVM in terms of accuracy for each setting and in turn both are ouperformed  by TC-QP-SVM. While there was no appreciable symmetry in either QP-SVM or LP-SVM, TC-QP-SVM exhibited significant variable and constraint reduction: the lifted problem was reduced to up to 78\% of the variables, resp., 70\% of the constraints of the ground problem, while computing the same labels and in turn accuracy. 

Qualitatively similar results 
were obtained in a final experiment on the two-moons dataset with 150 additional features, each drawn randomly from a Gaussian per example, and using the 4-nearest-neighbour graph as "citation network". 

\begin{figure*}[t]
\centering
\begin{subfigure}{1\textwidth}
\centering\tiny
\begin{lstlisting}[language=ampl,basicstyle=\ttfamily\tiny,basewidth  = {.5em,0.4em},
  columns    = flexible]
linked(I1, I2) = label(I1) & query(I2) & (cite(I1, I2) | cite(I2, I1)) # query for the transductive constraint
slacks   = sum{I in labeled(I)} slack(I);  coslacks = sum{I1, I2 in linked(I1, I2)} slack(I1,I2)  # inline definitions
# QUADRATIC OBJECTIVE, the main novelty compared to [Kersting et al., 2015]  
minimize: sum{J in feature(I,J)} weight(J)**2 + c1 * slack + c2 * coslack;         
subject to forall {I in labeled(I)}: labeled(I)*predict(I) >= 1 - slack(I); # push labeled examples to the correct side 
subject to forall {I in labeled(I)}: slack(I) >= 0;   # slacks are positive 	
# TRANSDUCTIVE PART: cited instances should have the same labels. 
subject to forall {I1, I2 in linked(I1, I2)}: labeled(I1) * predict(I2) >= 1 - slack(I1, I2);
subject to forall {I1, I2 in linked(I1, I2)}: coslack(I1, I2) >= 0;    #coslacks are positive
\end{lstlisting}
  \caption{TC-QP-SVM encoded in a novel QP extension of the relational LP language in~\citep{kersting2015}.\label{rlp:tclpsvm}}
\end{subfigure}

\begin{subfigure}{0.26\textwidth}
\centering
    \includegraphics[width=\textwidth]{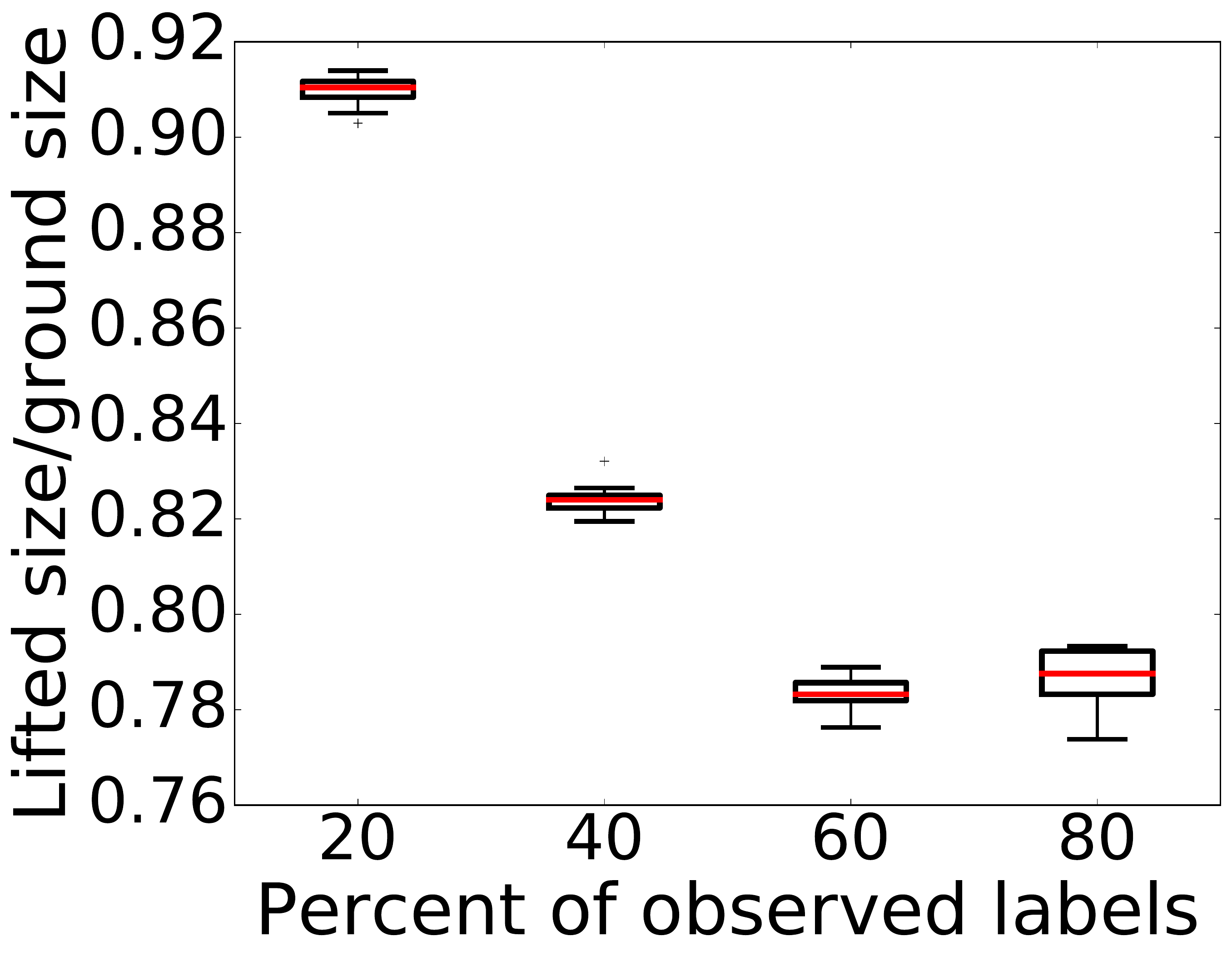}
    \caption{Variable compression\label{fig:var_compres}}
\end{subfigure}
\quad\quad\quad
\begin{subfigure}{0.26\textwidth}\centering

 \includegraphics[width=\textwidth]{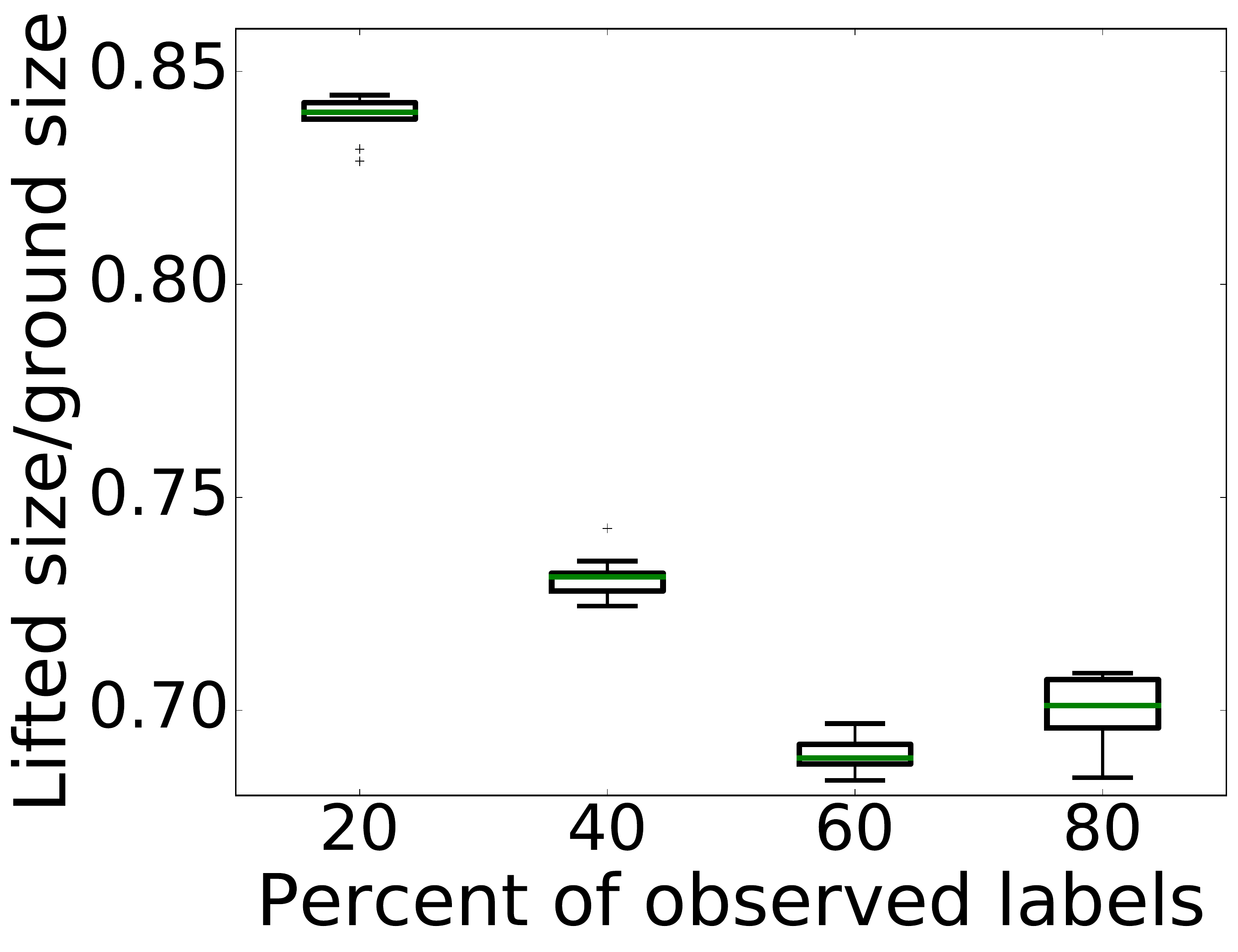}
     \caption{Constraint compression \label{constr_compr}}
\end{subfigure}
\quad\quad\quad
\begin{subfigure}{0.26\textwidth}\centering
\centering
\includegraphics[width=\textwidth]{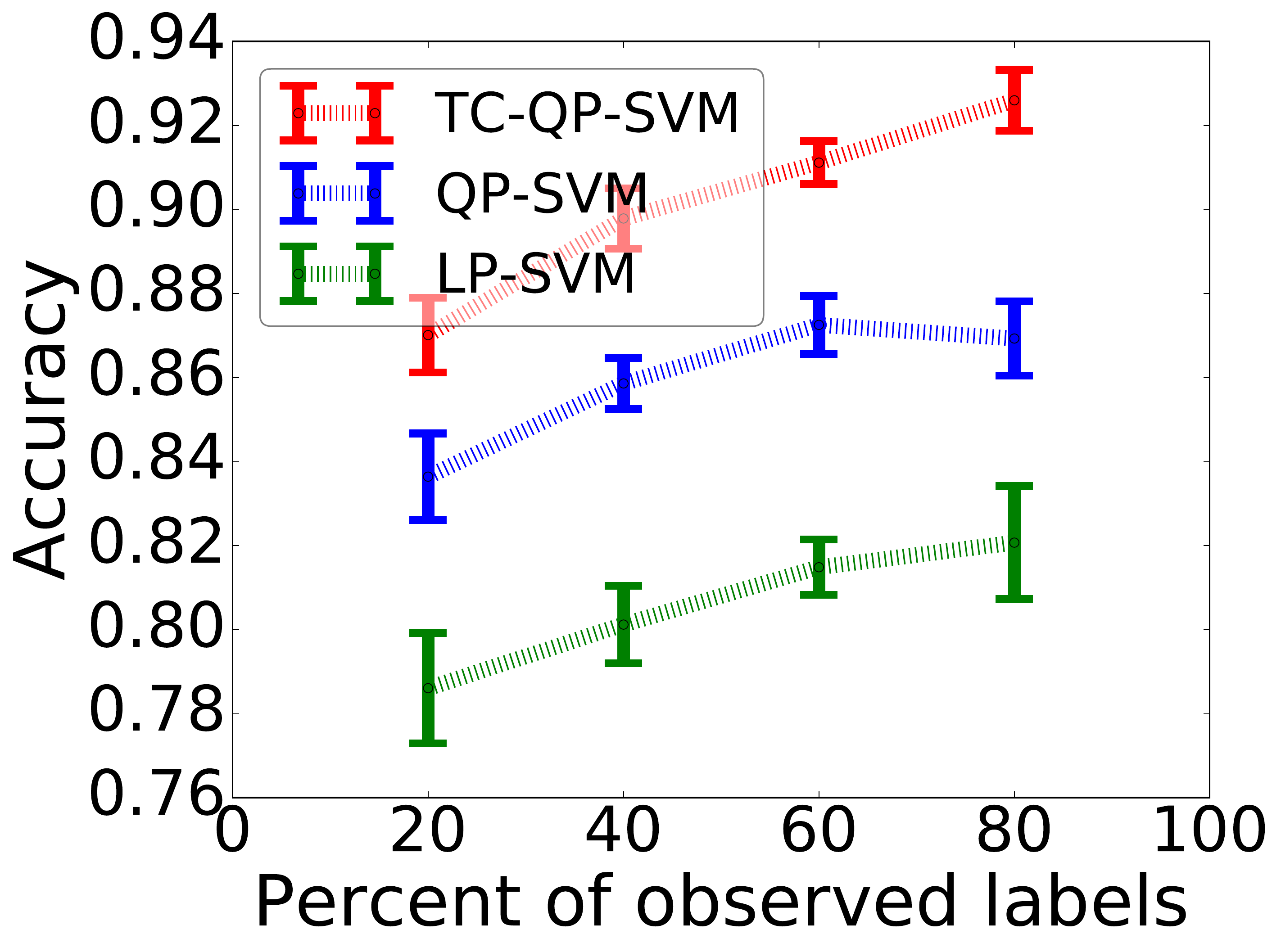}
\caption{Accuracy\label{fig:result_2}}
\end{subfigure}
\caption{Relational experiments on the Cora dataset. (Best viewed in color)\label{fig:result_1}}

\end{figure*}

Overall, the results of our experiments are clear evidence for an affirmative answer to question {\bf (Q)}.

\section{Conclusions} 
We have deepen the understanding of symmetries in machine learning and significantly extend the scope of lifted inference.
Specifically, we have introduced and studied a precise
mathematical definition of fractional symmetry of convex QPs. 
Using the tool of fractional automorphism, orbits of optimization 
variables are obtained, and lifted solvers materialize
as performing the corresponding optimisation problem
in the space of per-orbit optimization variables.
This enables the lifting of a large class of 
machine learning tasks and approaches.  We here instantiated this for SVMs by 
developing the first lifted solver
for SVMs and illustrating empirically its potential. In the future, other ML settings should be explored. One could also
deepen our theoretical results on more datasets, investigate the connection to other data reduction methods, develop
approximate WL graph kernels, and
move beyond convex QPs. Most significantly, our framework offers a mathematical foundation for symmetry-based machine learning~\citep{GensD14}.

{\bf Acknowledgements} The authors would like to thank the anonymous reviewers for their feedback. The work was partly supported by the 
DFG Collaborative Research Center SFB 876, project A6.

\bibliographystyle{plainnat}
\small
\bibliography{biblio}  
\end{document}